\documentclass{article}
\PassOptionsToPackage{numbers,compress}{natbib}

\usepackage{arxiv}
\usepackage{float}
\usepackage[utf8]{inputenc} 
\usepackage[T1]{fontenc}    
\usepackage{amsmath,amsfonts,amsthm,amssymb}
\usepackage{bm}
\usepackage{graphicx}
\usepackage{subcaption}
\usepackage{booktabs}
\usepackage{hyperref}
\usepackage{color}
\usepackage{microtype}
\usepackage{xspace}
\usepackage{nicefrac}
\usepackage{breqn}
\usepackage{dsfont}
\usepackage{enumitem}
\usepackage{natbib}
\usepackage{algorithm}
\usepackage{algorithmic}

\definecolor{darkblue}{rgb}{0.0,0.0,0.5}
\hypersetup{
	colorlinks = true,
	citecolor  = darkblue,
	linkcolor  = darkblue,
	citecolor  = darkblue,
	filecolor  = darkblue,
	urlcolor   = darkblue,
}

\newtheorem{theorem}{Theorem}

\newtheorem{prop}{Proposition}

\theoremstyle{definition}
\newtheorem{rmk}{Remark}
\newtheorem{defn}{Definition}

\newcommand{\dpp}{\textsc{Dpp}\xspace}
\newcommand{\dpps}{\textsc{Dpp}s\xspace}

\newcommand{\lowrank}{\textsc{LowRank}\xspace}
\newcommand{\dynCE}{\textsc{Dyn}\xspace}
\newcommand{\prodCE}{\textsc{Prod}\xspace}
\newcommand{\prodNCE}{\textsc{NCE}\xspace}
\newcommand{\expCE}{\textsc{Exp}\xspace}

\graphicspath{{./figs/}}

\def\argmax{\mathop{\rm argmax}}

\def\0{{\bm 0}}

\def\b{{\bm b}}

\def\v{{\bm v}}

\def\B{{\bm B}}
\def\C{{\bm C}}

\def\I{{\bm I}}

\def\L{{\bm L}}

\def\P{{\bm P}}

\def\U{{\bm U}}
\def\V{{\bm V}}

\def\X{{\bm X}}

\def\Z{{\bm Z}}

\def\Lambda{\boldsymbol{\lambda}}

\def\Acal{\mathcal{A}}

\def\Pcal{\mathcal{P}}

\def\Tcal{\mathcal{T}}

\def\Ycal{\mathcal{Y}}

\def\Rbb{\mathbb{R}}

\def\wo{\backslash}

\title{Learning Determinantal Point Processes by \\Corrective Negative Sampling}

\author{\name Zelda Mariet$^*$ \email{zelda@csail.mit.edu}\\
  \name Mike Gartrell$^\dagger$ \email{m.gartrell@criteo.com}\\
  \name Suvrit Sra$^*$ \email{suvrit@mit.edu}\\
  \addr{$*$ Massachusetts Institute of Technology}\\
  \addr{$\dagger$ Criteo AI Lab}
}

\begin{document}
\maketitle

\begin{abstract}
  Determinantal Point Processes (\dpps) have attracted significant interest from
  the machine-learning community due to their ability to elegantly and tractably
  model the delicate balance between quality and diversity of sets.  \dpps are
  commonly learned from data using maximum likelihood estimation (MLE).  While
  fitting observed sets well, MLE for \dpps may also assign high likelihoods to
  unobserved sets that are far from the true generative distribution of the
  data. To address this issue, which reduces the quality of the learned model,
  we introduce a novel optimization problem, \emph{Contrastive Estimation (CE)},
  which encodes information about ``negative'' samples into the basic learning
  model. CE is grounded in the successful use of negative information in
  machine-vision and language modeling. Depending on the chosen negative
  distribution (which may be static or evolve during optimization), CE assumes
  two different forms, which we analyze theoretically and experimentally. We
  evaluate our new model on real-world datasets; on a challenging dataset, CE
  learning delivers a considerable improvement in predictive performance over a
  \dpp learned without using contrastive information.
\end{abstract}


\vskip 1em
\section{Introduction}
Careful selection of items from a large collection underlies many machine learning applications. Notable examples include recommender systems, information retrieval and automatic summarization methods, among others. Typically, the selected set of items must fulfill a variety of application specific requirements---e.g., when recommending items to a user, the \emph{quality} of each selected item is important. This quality must be, however, balanced by the \emph{diversity} of the selected items to avoid redundancy within recommendations.

But balancing quality with diversity is challenging: as the collection size grows, the number of its subsets grows exponentially. A model that offers an elegant, tractable way to achieve this balance is a Determinantal Point Process (\dpp). Concretely, a \dpp models a distribution over subsets of a ground set $\Ycal$ that is parametrized by a semi-definite matrix $\L \in \Rbb^{|\Ycal| \times |\Ycal|}$, such that for any $A \subseteq \mathcal Y$,
\begin{equation}
  \label{eq:dpp}
  \Pr(A) \propto \det(\L_A),
\end{equation}
where $\L_A=[\L_{ij}]_{i,j\in A}$ is the submatrix of $\L$ indexed by $A$. Informally, $\det(\L_A)$ represents the volume associated with subset $A$, the diagonal entry $L_{ii}$ represents the importance of item $i$, while entry $L_{ij} = L_{ji}$ encodes similarity between items $i$ and $j$. Since the normalization constant of~\eqref{eq:dpp} is simply $\sum_{A\subseteq \Ycal}\det(\L_A)=\det(\L+\I)$, we have $\Pr(A)=\det(\L_A)/\det(\L+\I)$, which suggests why \dpps may be tractable despite their exponentially large sample space.

The key object defining a \dpp is its kernel matrix $\L$. This matrix may be
fixed \emph{a priori} using domain knowledge~\citep{borodin2009}, or as is more
common in machine learning applications, learned from observations using maximum
likelihood estimation (MLE)~\citep{gillenwater14,mariet15}. However, while
fitting observed subsets well, MLE for \dpps may also assign high likelihoods to
unobserved subsets far from the underlying generative
distribution~\citep{chao15}, since MLE causes the \dpp model to maximize the
determinantal volume of observed subsets without explicitly minimizing the
volume of unobserved subsets.  Therefore, the volume of unobserved subsets may
be larger than expected, and MLE-based \dpp models may thus have modes
corresponding to subsets that are close in likelihood, yet differ in how close
they are to the true data distribution. Such confusable modes reduce the quality
of the learned model, hurting predictions, as shown in Figure~\ref{fig:motivation}.  

Such concerns when learning generative models over huge sample spaces are not limited to the area of subset-selection: applications in image and text generation have been the driving force in developing techniques for generating high-quality samples. Among their innovations, a particularly successful technique uses generated samples as ``negative samples'' to train a discriminator, which in turn encourages generation of more realistic samples; this is the key idea behind the Generative Adversarial Nets (GANs) introduced in~\citep{goodfellow14}. 

These observations motivate us to investigate the use of \dpp-generated samples with added perturbations as \emph{negatives}, which we then incorporate into the learning task to improve the modeling power of \dpps. Intuitively, negative samples are those subsets that are far from the true data distribution, but to which the \dpp erroneously assigns high probability. As there is no closed form way to generate such idealized negatives, we approximate them via an external ``negative distribution''.

More precisely, we introduce a novel \dpp learning problem that incorporates
samples from a negative distribution into traditional MLE.  Our approach reduces
the confusable mode issue associated with MLE for \dpps by augmenting MLE with a
term that explicitly minimizes the volume of unobserved subsets that are far
from the true data distribution (negative samples). While the focus of our work
is on generating the negative distribution \emph{jointly} with $\L$, we also
investigate outside sources of negative information. Ultimately, our formulation
leads to an optimization problem harder than the original \dpp learning problem;
we show that even approximate solutions greatly improve the performance of the
\dpp model when evaluated on concrete tasks, such as identifying the best item
to add to a subset of chosen objects (\emph{basket-completion}) and
discriminating between held-out test data and randomly generated subsets.

\begin{figure*}[t]
  \centering
  \begin{subfigure}{.49\textwidth}
    \centering
    \includegraphics{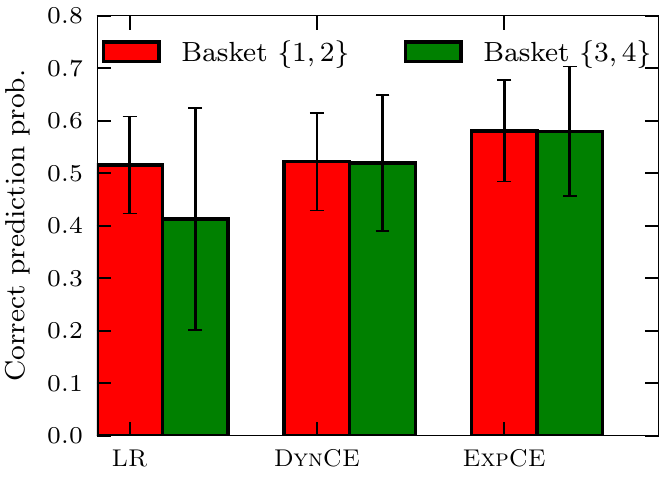}
    \caption{Probabilities of making the right prediction}
  \end{subfigure}
  \begin{subfigure}{.49\textwidth}
    \centering
    \includegraphics{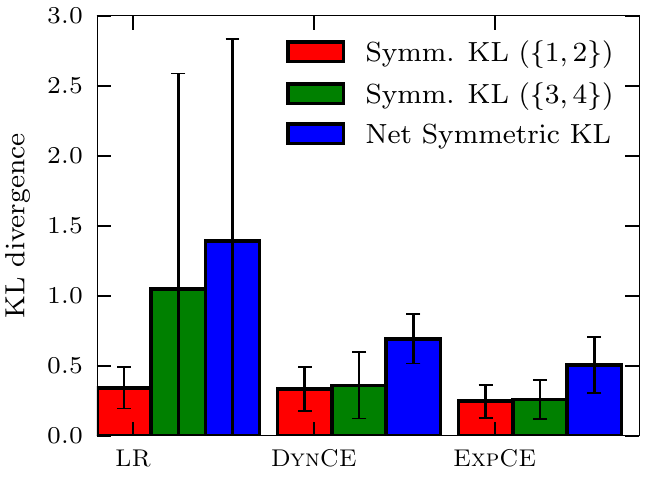}
    \caption{KL divergences}
  \end{subfigure}
\caption{Results for experiments on a synthetic toy dataset. This toy dataset
was generated by replicating the baskets \{1, 2\} and \{3, 4\} 1000 times each.
We randomly select 80\% of this dataset for training, and 20\% for test. We
train each model to convergence, and then compute the next-item predictive
probabilities for each unique pair, along with the symmetric KL divergence
(over areas of shared support) between the predictive and empirical next-item distributions. Net symmetric KL divergence is computed by adding the symmetric
KL divergences for each of the two unique baskets.  Experiments were run 10
times, with  $|\mathcal A^+|/|\mathcal A^-|$ set to the optimal value for each
model; $\alpha$ is set to its optimal \lowrank value.  See
Section~\ref{sec:learning-with-negatives} for details on the \dynCE and \expCE
negative sampling models. The LR (low-rank \dpp) model
assigns relatively high probabilities to modes that represent
incorrect predictions: as Maximum Likelihood Estimation learning teaches the model to maximize the volume of observed subsets without explicitly minimizing the volume of unobserved subsets, volumes of unobserved subsets may be larger than expected. The \dynCE and \expCE methods we introduce in Section~\ref{sec:learning-with-negatives} reduce this confusable
mode issue, resulting in predictive distributions that are closer to the true
distribution and much smaller variances.}
\label{fig:motivation}
\end{figure*}

\paragraph{Contributions.} To our knowledge, this work is the first theoretical or empirical investigation of augmenting the \dpp learning problem with negative information.
\begin{itemize}
\item[--] Our first main contribution is the Contrastive Estimation (CE) model, which incorporates negative information through \emph{inferred negatives} into the learning task. 
\item[--] We introduce static and dynamic models for CE and discuss the theoretical and practical trade-offs of such choices.  Static models leverage information that does not evolve over time, whereas dynamic models draw samples from a negative distribution that depends on the current model's parameters; dynamic CE posits an optimization problem worthy of independent study.
\item[--] We show how to learn CE models efficiently, and furthermore show that the complexity of conditioning a \dpp on a chosen sample can be brought from $\mathcal O(|\Ycal|^2)$ to essentially $\mathcal O(|\Ycal|)$. This helps dynamic CE and removes a major bottleneck in computing next-item predictions for a set.
\end{itemize}
Using findings obtained from extensive experiments conducted on small datasets,
we show on a large dataset that CE learning significantly improves the modeling
power of \dpps: CE learning improves \dpp performance for next-item basket
completion, as well as \dpp discriminative power, as evaluated by the model's
ability to distinguish held-out test data from randomly generated subsets. 

We present a review of related work in Section~\ref{sec:related-work}. In Section~\ref{sec:learning-with-negatives}, we introduce Contrastive Estimation and its dynamic and static variants. We discuss how the CE problem can be optimized efficiently in Section~\ref{sec:efficient}, as well as how \dpp conditioning for basket-completion predictions can be performed with improved complexity. In Section~\ref{sec:experiments}, we show that  CE learning leads to remarkable empirical improvements of \dpp performance metrics.

\section{Background and related work}
\label{sec:related-work}

First introduced to model fermion behavior by~\citet{macchi75}, \dpps have
gained popularity due to their elegant balancing of quality and subset
diversity. \dpps are studied both for their theoretical
properties~\citep{kuleszaBook,borodin2009,affandi14,kuleszaThesis,gillenwater-thesis,decreuse,lavancier15}, which include fast sampling~\citep{pmlr-v40-Rebeschini15,pmlr-v48-lih16,pmlr-v49-anari16},
and for their machine learning applications: object retrieval~\citep{affandi14},
summarization~\citep{lin12,chao15}, sensor placement~\citep{krause08},
recommender systems~\citep{gartrell16}, neural network
compression~\citep{mariet16}, and minibatch selection~\citep{zhang17}. 

\citet{gillenwater14} study \dpp kernel learning via EM, while \citet{mariet15}
present a fixed-point method. \dpp kernel learning has leveraged
Kronecker~\citep{mariet16b} and low-rank~\citep{dupuy16,gartrell17} structures.
Learning guarantees using \dpp graph properties are studied
in~\citep{urschel17}. Aside from~\citet{tschiatschek16,djolonga16}, who learn a
Facility LocatIon Diversity (FLID) distribution (as well as more complex FLIC
and FLDC models) by contrasting it with a ``negative'' product distribution,
little attention has been given to using negative samples to learn richer
subset-selection models.

Nonetheless, leveraging negative information is a widely used in other
applications. In object detection, negative mining corrects for the skewed
simple-to-difficult negative distribution by training the model on its false
positives~\citep{sung96,cavenet14,shrivastava16}. In language modeling, Noise
Contrastive Estimation (NCE)~\citep{gutmann12}, which tasks the model with
distinguishing positive samples from generated negatives, was first applied
in~\citep{mnih12} and has been instrumental in Word2Vec~\citep{mikolov13}. Since
then, variants using adaptive noise~\citep{chen17} have been introduced. NCE is
also the method used by~\citet{tschiatschek16} for subset-selection.

An alternate approach to negative samples within submodular language models was
introduced as Contrastive Estimation in~\citep{smith05,smith05b}. Negative
sampling is also used in GANs~\citep{goodfellow14}, where a generator network
competes with a discriminative network which distinguishes between positives and
generated negatives.  An adversarial approach to Contrastive Estimation has been
recently introduced in~\citep{bose2018adversarial}, where ideas from GANs for
discrete data are used to implement an adversarial negative sampler that
augments a conventional negative sampler.


\section{Learning \dpps with negative samples}
\label{sec:learning-with-negatives}
Motivated by the similarities between \dpp learning and crucial structured prediction  problems in other ML fields, we introduce an optimization problem that leverages negative information. We refer to this problem as Contrastive Estimation (CE) due to its ties to a notion discussed in~\citep{smith05}.
 \subsection{Contrastive Estimation}
 In conventional \dpp learning, we seek to maximize determinantal volumes of sets drawn from the true distribution $\mu$ (that we wish to model), by 
 solving the following MLE problem, where samples in the training set $\Acal^+$ are assumed to be drawn i.i.d.:
 \begin{align}
   \label{eq:mle}
   \text{Find } \L \in \argmax_{\L \succeq 0} \phi_{\text{MLE}}(\L)  \triangleq& \tfrac{1}{|\Acal^+|}\sum_{A \in \Acal^+} \log \det (\L_A) - \log \det (\L+\I). 
 \end{align}
We augment problem~\eqref{eq:mle} to incorporate additional information from a \emph{negative} distribution $\nu$, which we wish to have the \dpp distribution move away from. The ensuing optimization problem is the main focus of our paper.
\begin{defn}[Contrastive Estimation] Given a training set of positive samples $\Acal^+$ on which $\phi_{\text{MLE}}$ is defined and a negative distribution $\nu$ over $2^{\Ycal}$, we call \emph{Contrastive Estimation} the  problem
\begin{equation}
  \label{eq:ce}
  \text{Find }\L \in \argmax_{\L \succeq 0} \phi_{\text{CE}}(\L) \triangleq\ \phi_{\text{MLE}}(\L) - \mathbb E_{A \sim \nu} [\log \Pcal_\L(A)], 
\end{equation}
where we write $\Pcal_{\L}(A) \equiv \det(\L_A)/\det(\L+\I)$.
\end{defn}

The expectation can be approximated by drawing a set of samples $\Acal^-$ from $\nu$:  $\phi_{\text{CE}}$ then becomes\footnote{With a slight abuse of notation, we continue writing $\phi_{\text{CE}}$ despite the sample approximation to $\mathbb{E}_{A\sim\nu}[\cdot]$.}
\begin{align}
  \label{eq:ce2}
  \phi_{\text{CE}}(\L) =& \tfrac{1}{|\Acal^+|} \sum_{A \in \Acal^+} \log \Pcal_\L(A) - \tfrac{1}{|\Acal^-|} \sum_{A \in \Acal^-} \log \Pcal_\L(A)
\end{align}
If $|\Acal^-|=0$, the CE objective \eqref{eq:ce} reduces to $\phi_{\text{MLE}}$.
Conversely,  $\phi_{\text{MLE}}$ can be viewed as a sample-based approximation of the value $\mathbb E_{A \sim \mu}[\log \P_\L(A)]$, where $\mu$ is the true distribution generating the samples in $\Acal^+$.
Interestingly, another reformulation of \eqref{eq:ce} suggests an even broader class of \dpp kernel learning: indeed, let $y_A$ be $\frac 1 {|\Acal^+|}$ (resp. $-\frac 1 {|\Acal^-|}$) for $A \in \Acal^+$ (resp. $\Acal^-$), and define  \[\Acal = \{(y_A, A) : A \in \Acal^+\} \cup \{(y_A, A) : A \in \Acal^-\},\]
where the $y_A$ should be viewed as belonging in $\{-1,1\}$ with an additional normalization coefficient. Then, we can rewrite equation~\eqref{eq:ce2} in the following form
\begin{equation}
  \label{eq:2}
  \phi_{\text{CE}}(\L)=\!\!\!\sum_{(y_A,A) \in \Acal} y_A \Big[\log \det \L_A - \det(\L+\I)\Big].
\end{equation}
Formulation~\eqref{eq:2} suggests the use of a broader scope of continuous labels $y_A$; we do not cover this variation in the present work, but note that \eqref{eq:2} permits the use of \emph{weighted} samples for learning.

\begin{rmk}
Compared to the traditional Noise Contrastive Estimation (NCE) approach, which requires full knowledge of the negative distribution, CE does not suffer any such limitation: we only require an estimate of $\mathbb E_\nu[\log \Pcal_\L(A)]$.
\end{rmk}

\begin{rmk}
  \label{rm:fun}
  Eq.~\eqref{eq:ce} can be made to go to $+\infty$ with pathological negative samples (i.e. $\Pcal_\L(A^-) = 0$); hence, choosing the negative distribution is a crucial concern for CE. In practice, we do not observe this pathological behavior (cf. Section~\ref{sec:experiments}).
\end{rmk}

\begin{rmk}
  CE is a non-convex optimization problem, and thus admits the same guarantees
  as \dpp MLE learning when learned using Stochastic Gradient Ascent with decreasing
  step sizes; however, the convergence rate will depend on the choice of $\nu$.
\end{rmk}

Indeed, to fully specify the CE problem one must first choose the negative distribution $\nu$, or equivalently, choose a procedure to generate  negative samples to obtain \eqref{eq:ce2}. We consider below two classes of distributions $\nu$ with considerably different ramifications: dynamic and static negatives; their analysis is the focus of the next two sections.

\subsection{Dynamic negatives}
\label{sec:negative-sampling}
In most applications leveraging negative information (e.g., negative mining, GANs), negative samples evolve over time based on the state of the learned model. We call any $\nu$ that depends on the state of the model a \emph{dynamic negative distribution}: at iteration $k$ of the learning procedure with kernel estimate $\L_k$, we use a $\nu$ parametrized by $\L_k$.

More specifically, we focus on the setting where negative samples themselves are generated by the current \dpp, with the goal of reducing overfitting. Given a positive sample $A^+$, we generate a negative $A^-$ by replacing $i \in A^+$ with $j$ that yields a high probability $\Pcal_{\L_k}(A^+ \wo \{i\} \cup \{j\})$ (Alg.~\ref{alg:dyn}). We generate the samples probabilistically rather than via mode maximization so that a sample $A^+$ can lead to different $A^-$ negatives when we generate more negatives than positives.
\begin{algorithm}[h]
  \caption{Generate dynamic negative}
  \label{alg:dyn}
  \begin{algorithmic}
    \STATE {\bfseries Input:} Positive sample $A^+$, current kernel $\L_k$
    \STATE Sample $i \in A^+$ prop. to its empirical probability in $\Acal^+$
    \STATE $A^- := A^+ \wo \{i\}$
    \STATE Sample $j$ w.p. proportional to $\Pcal_{\L_k}(A^- \cup \{j\})$
    \STATE $A^- \leftarrow A^- \cup \{j\}$
    \item[] \textbf{return} $A^-$
  \end{algorithmic}
\end{algorithm}


As $\nu$ evolves along with $\L_k$, the second term of  $\phi_{\text{CE}}$ acts as a moving target that must be continuously estimated during the learning procedure. For this reason, we choose to optimize $\phi_{\text{CE}}$ by a two-step procedure described in Alg.~\ref{alg:em}, similarly to an alternating maximization approach such as EM. 
\begin{algorithm}[h]
  \caption{Optimizing dynamic CE}
  \label{alg:em}
  \begin{algorithmic}
    \STATE {\bfseries Input:} Positive samples $\Acal^+$, initial kernel $\L_0$, maxIter.
    \STATE $k \leftarrow 1$
    \WHILE {$k++ < $ maxIter {\bfseries and} not converged}
    \STATE $\Acal^- \leftarrow \textsc{GenerateDynamicNegatives}(\L_k$, $\Acal^+$)
    \STATE $\L_{k+1} \leftarrow \textsc{OptimizeCE}(\L_k, \Acal^+, \Acal^-)$
    \ENDWHILE
    \item[] \textbf{return} $\L_k$
  \end{algorithmic}
\end{algorithm}

Note that this approach bears strong similarities with GANs, in which both the generator and discriminator evolve during training (dynamic negatives also appear in a discussion by~\citet{goodfellow14b} as a theoretical tool to analyze the difference between NCE and GANs).

Once the generated negative $A^-$ has been used in an iteration of the optimization of $\phi_{\text{CE}}$, it is less likely to be sampled again.\footnote{If $A^-$ happens to be a false negative (i.e. appears in $\mathcal A^+$), $A^-$ will be comparatively sampled more frequently as a positive, and so will contribute on average as a positive sample. Additional precautions such as the ones mentioned in~\citep{bose2018adversarial} can also be leveraged if necessary.} Crucially, such dynamic negatives also avoid the problem alluded to in Remark~\ref{rm:fun}, since by construction they have a non-zero probability under $\Pcal_{\L_k}$ at iteration $k$.

\subsection{Static negatives}
Conversely, we can simplify the optimization problem by considering a \emph{static} negative distribution: $\nu$ does not depend on the current kernel estimate.
A considerable theoretical advantage of static negatives lies in the simpler optimization problem: given a static negative distribution $\nu$, the optimization objective $\phi_{\text{CE}}$ does not evolve during training, and is amenable to a simple invocation of stochastic gradient descent~\citep{bottou1999}.

\begin{theorem}
  Let $\nu$ be a static distribution over $2^{\mathcal Y}$ and let $k > 0$ be such that $k \ge \max \{|S|: S \in \mathcal A^+ \cup \textup{\,supp}(\nu)\}$. Let $K$ be a bounded subspace of all  $|\Ycal| \times |\Ycal|$ positive semi-definite matrices of rank $k$. Projected stochastic gradient ascent applied to the CE objective with negative distribution $\nu$ and space $K$ with step sizes $\eta_i$ such that $\sum \eta_i = \infty$, $\sum \eta_i^2 < \infty$ will converge to a critical point.
\end{theorem}

Note, however, that such distributions may suffer from the fundamental theoretical issue in Rem.~\ref{rm:fun}, and hence careful attention must be paid to ensure that the learning algorithm does not converge to a spurious optimum that assigns a probability $\P_\L(A) = 0$ to $A \in \Acal^-$. In practice, we observed that the local nature of stochastic gradient ascent iterations was sufficient to avoid such behavior.

Let us now discuss two classical choices for fixed $\nu$.
\paragraph{Product negatives.} A common choice of negative distribution in other machine learning areas is the \emph{product distribution}, which is the standard ``noise'' distribution used in NCE. It is  defined by
\begin{equation}
  \label{eq:3}
  \nu(A) = \prod\nolimits_{i \in A} \hat p(i)\prod\nolimits_{i \not\in A} (1-\hat p(i))
\end{equation}
where $\hat p(i)$ is the empirical probability of $\{i\}$ in $\Acal^+$. Although~\citep{mikolov13} reports better results by raising the $\hat p$ to the power $\tfrac 34$, we did not observe any improvements when using exponentiated power distributions; for this reason, by \emph{product negatives}, we always indicate the baseline distribution \eqref{eq:3}. 

The product distribution is in practice a mismatch for \dpps, as it lacks the negative association property of \dpps which  enables them to model the repulsive interactions between similar items\footnote{\dpps belong to the family of \emph{Strongly Rayleigh} measures, which have been shown to verify a broad range of negatively associated properties; we refer the interested reader to the fascinating work~\citep{pemantle-00,borcea-09,borcea-09b,borcea-09c,borcea-09d,borcea-10}.}. 

\paragraph{Explicit negatives.} Alternatively, we may have prior knowledge of a class of subsets that our model should \emph{not} generate. For example, we might know that items $i$ and $j$ are negatively correlated and hence unlikely to co-occur. We may also learn via user feedback that some generated subsets are inaccurate. We refer to negatives obtained using such outside information as \emph{explicit negatives}.

A fundamental advantage of explicit negatives is that they allow us to incorporate prior knowledge and user feedback as part of the learning algorithm. The ability to incorporate such information, to our knowledge, is in itself a novel contribution to \dpp learning.

Although such knowledge may be costly and/or only available at rare intervals,  a form of continuous learning that would regularly update the state of our prior knowledge (and hence $\nu$) would bring the explicit negative distribution into the realm of dynamic distributions, as described by Alg.~\ref{alg:em}.

\section{Efficient learning and prediction}
\label{sec:efficient}
We now describe how the Contrastive Estimation problem for \dpps can be optimized efficiently.
In order to efficiently generate dynamic negatives, which rely on \dpp conditioning, we additionally generalize the dual transformation leveraged in~\citep{osogami18} to speed up basket-completion tasks with \dpps. This speed-up impacts the broader use of \dpps, outside of CE learning.

\subsection{Optimizing $\phi_{\text{CE}}$}
We propose to optimize the CE problem by exploiting a low-rank factorization of the kernel, writing $\L = \V\V^\top$, where $\V \in \mathbb R^{M \times K}$ and $K \le M$ is the rank of the kernel, which is fixed {\it a priori}.

This factorization ensures that the estimated kernel remains positive semi-definite, and  enables us to leverage the low-rank computations derived in~\citep{gartrell17} and refined in~\citep{osogami18}. Given the similar forms of the MLE and CE objectives, we use the traditional stochastic gradient ascent algorithm introduced by~\citep{gartrell17} to optimize~\eqref{eq:ce}. In the case of dynamic negatives, we re-generate $\Acal^-$ after each gradient step;  less frequent updates are also possible if the negative generation algorithm is very costly.

We furthermore augment $\phi_{\text{CE}}$ with a regularization term $R(\V)$, defined as
\[R(\V) = \alpha\sum\nolimits_{i=1}^M \frac 1 {\mu_i} \|\v_i\|_2^2,\]
where $\mu_i$  counts the occurrences of $i$ in the training set, $\v_i$ is the corresponding row vector of $\V$ and $\alpha > 0$ is a tunable hyperparameter. Note that this is the same regularization as introduced in~\citep{gartrell17}.
This regularization tempers the strength of $\|\v_i\|_2$, a term interpretable as to the popularity of item~$i$ \citep{kuleszaBook,gillenwater-thesis}, based on its \emph{empirical} popularity $\mu_i$. Experimentally, we observe that adding $R(\V)$ has a strong impact on the predictive quality of our model.

The reader may wonder if other approaches to \dpp learning are also applicable to the CE problem.
\begin{rmk}
  Gradient ascent algorithms require that the estimate $\L$ be projected onto the space of positive semi-definite matrices; however, doing so can lead to almost-diagonal kernels~\citep{gillenwater14} that cannot model negative interactions. Riemannian gradient ascent methods were considered, but deemed too computationally demanding by~\citep{mariet15}. Furthermore, the update rule for the fixed-point approach in~\citep{mariet15} does not admit a closed form solution for CE, rendering it impractical (App.~\ref{app:picard}).
\end{rmk}
The low-rank formulation allows us to apply CE (as well as NCE, as discussed in Section~\ref{sec:experiments}) to learn large datasets such as the Belgian retail supermarket dataset (described in Section~\ref{sec:experiments}) without prohibitive learning runtimes. We show below that by leveraging the idea described in~\citep{osogami18}, the low-rank formulation can also lead to additional speed ups during prediction.

\subsection{Efficient conditioning for predictions}
Dynamic negatives rely upon conditioning a \dpp on a chosen sample $A$ (see Alg.~\ref{alg:dyn}: $\Pcal_{\L_k}(A^- \cup \{j\})$ can be efficiently computed for all $j$ by a preprocessing step that conditions $\L_k$ on set $A^-$). For this reason, we now describe how low-rank \dpp conditioning can be significantly sped up.

In~\citep{gartrell17},  conditioning has a cost of $\mathcal O(K|\bar{A}|^2 + |A|^3)$, where $\bar{A} = \Ycal - A$.  Since
$|\Ycal|\, \gg \!|A|$ for many datasets, this represents a significant bottleneck
for conditioning and computing next-item predictions for a set.  We show here
that this complexity can be brought down significantly.
\begin{prop}
  Given $A \subseteq \{1, \ldots, M\}$ and a \dpp of rank $K$ parametrized by $\V$, where $\L = \V\V^\top$, we can derive the conditional marginal probabilities in the \dpp parametrization $\L^A$ in only $\mathcal O(K^3 + |A|^3 + K^2 |A|^2 + |\bar{A}| K^2)$ time.
\end{prop}
\begin{proof}
  Let $\V$ be the low-rank parametrization of the \dpp kernel ($\L = \V \V^\top$) and $A \subseteq \mathcal Y$. As in~\cite{gillenwater-thesis}, we first compute the dual kernel $\C = \B^\top\B$, where $\B =
\V^\top$.  We then compute
\begin{equation*}
 \C^A = (\B^A)^\top \B^A = \Z^A \C \Z^A,
\end{equation*}
with $\Z^A = \I - \B_A(\B_A^\top \B_A)^{-1} \B_A^\top$, and where $\C^A$ is the \dpp kernel conditioned on the event that all items in $A$ are observed, and $\B_A$ is the restriction of $\B$ to the rows and columns indexed by $A$.

Computing $\C^A$ costs $\mathcal O(K^3 + |A|^3 + K^2 |A|^2)$. Next,
following~\cite{kuleszaBook}, we eigendecompose $\C^A$ to compute the conditional
(marginal) probability $P_i$ of every possible item $i$ in $\bar{A}$:
\begin{equation*}
P_i = \sum\nolimits_{n=1}^K \tfrac{\lambda_n}{\lambda_n + 1}\left( \tfrac{1}{\sqrt{\lambda_n}} \b_i^A \hat{\v}_n  \right)^2
\end{equation*}
where $\b_i^A$ is column vector for item $i$ in $\B^A$ and $(\lambda_n$, $\hat \v_n)$ are an
eigenvalue/vector of $\C^A$.

The computational complexity for computing the eigendecomposition is $\mathcal O(K^3)$, and
computing $P_i$ for all items in $\bar{A}$ costs $\mathcal O(|\bar{A}| K^2)$. Therefore,
we have an overall computational complexity of $\mathcal O(K^3 + |A|^3 + K^2 |A|^2 +
|\bar{A}| K^2)$ for computing next-item conditionals/predictions for the
low-rank \dpp using the dual kernel, which is significantly superior to the
typical cost of $\mathcal O(K|\bar{A}|^2 + |A|^3)$.
\end{proof}
As in most cases $K \ll |\bar{A}|$, this represents a substantial improvement, allowing us condition
in time essentially linear in the size of the item catalog.

\section{Experiments}
\label{sec:experiments}
\begin{table*}[t]
  \centering
  \scalebox{.75}{
    \begin{tabular}{cccc}
      \multicolumn{4}{c}{(a) UK dataset}\\
      \toprule
      ~ & ~ & \multicolumn {2}{c}{Improvement over \lowrank} \\ \cline{3-4}\\[-.2cm]
      Metric & \lowrank & \expCE & \dynCE \\
      \midrule
      MPR & 80.07 & \bf 3.75 $\pm$ 0.16 & \bf 3.74 $\pm$ 0.16 \\
      AUC & 0.57297 & \bf 0.41465 $\pm$ 0.01334 & \bf 0.41467 $\pm$ 0.01339 \\
      \bottomrule
    \end{tabular}
  }
  \scalebox{.75}{
    \begin{tabular}{ccc}
      \multicolumn{3}{c}{(b) Belgian dataset}\\
      \toprule
      ~  & \multicolumn {2}{c}{Improvement over \lowrank} \\ \cline{2-3}\\[-.2cm]
      \lowrank & \expCE & \dynCE \\
      \midrule
      79.42 & \bf 9.58 $\pm$ 0.15 & \bf 9.64 $\pm$ 0.13 \\
      0.6162 & \bf 0.3705 $\pm$ 8.569e-5 & \bf 0.3702 $\pm$ 4.447e-5 \\
      \bottomrule
    \end{tabular}
  }
  \caption{Results over the UK and Belgian datasets. Both explicit and dynamic CE obtain statistically significant improvements in MPR and AUC metrics, confirming that CE learning enhances recommender value of the model and its ability to distinguish data drawn from the target distribution from fake samples. The impact on precisions@$k$ metrics is not reported as we did not observe statistically significant deviations from LowRank performance.}
\end{table*}

We run next-item prediction and AUC-based classification
experiments\footnote{All code is implemented in Julia and will be made
publicly available upon publication.} on two recommendation datasets for \dpp
evaluation: the UK retail dataset~\citep{lsbupr1492}, which after clipping all
subsets to a maximum size\footnote{This allows us to use a low-rank matrix
factorization for the \dpp that scales well in terms of train and prediction
time.} of 100, contains 4070 items and 20059 subsets,
and the Belgian Retail Supermarket
  dataset\footnote{\url{http://fimi.ua.ac.be/data/retail.pdf}}, which 
  contains 88,163 subsets, of a total of 16,470 unique
  items~\citep{brijs99,brijs03}.
We compare the following Contrastive Estimation approaches:
\begin{itemize}
  \item[--] \expCE: explicit negatives learned with CE. As to our knowledge there are no datasets with explicit negative information, we generate approximations of explicit negatives by removing one item from a positive sample and replacing it with the least likely item (Algorithm~\ref{alg:dyn-neg}).
  \item[--] \dynCE: dynamic negatives learned with CE.
\end{itemize}
As our work revolves around improving \dpp performance, we focus on the two following baselines, which are targeted to learning \dpp parametrizations from data:
\begin{itemize}
    \item[--] \prodNCE: Noise Contrastive Estimation using product negatives.
    \item[--] \lowrank: the standard low-rank \dpp stochastic gradient ascent algorithm from~\citep{gartrell17}.
\end{itemize}
NCE learns a model by contrasting $\Acal^+$ with negatives drawn from a ``noisy'' distribution $p_n$, training the model to distinguish between sets drawn from $\mu$ and sets drawn from $p_n$. NCE has gained popularity due to its ability to model distributions $\mu$ with untractable normalization coefficients, and has been shown to be a powerful technique to improve submodular recommendation models~\citep{tschiatschek16}. NCE learns by maximizing the following conditional log-likelihood:
\begin{align}
  \label{eq:nce}
  \phi_{\text{\tiny NCE}}(\L) = & {\sum_{A \in \Acal^+} \log P(A \in \Acal^+ \mid A)} + {\sum_{A \in \Acal^-} \log P(A \in \Acal^- \mid A)}. \end{align}
The key difference between NCE and CE lies in how negative information is used: whereas CE learns to assign a low probability to negative subsets, NCE's task is more indirect, learning to distinguish positive from negative examples. As a consequence of NCE's objective function~(Eq. \ref{eq:nce}), NCE requires knowledge of the distribution of negative samples, making it difficult to apply when explicit negative samples are available, but not the form of their distribution.

In our experiments, we learn the NCE objective with stochastic gradient ascent for our low-rank model, since $\nabla \log \Pr(A \in \Acal^* | A, {\V\V^\top})$ is given by
\begin{equation}
  \label{eq:gradcne}
  \Big(\epsilon^* - \Big({1 + \frac{|\Acal^-|}{|\Acal^+|} \frac{p_n(A)}{ \Pcal_{\V\V^\top}(A)}}\Big)^{-1} \Big) \nabla_\V \log \Pcal_{\V\V^\top}(A).
\end{equation}
where $\epsilon^*=1$ if $\Acal^* =  \Acal^+$ and 0 otherwise.

\begin{algorithm}[h]
  \caption{Approximate explicit negative generation}
  \label{alg:dyn-neg}
  \begin{algorithmic}
    \STATE {\bfseries input:} Positive sample $A^+$
    \STATE Sample $i\neq j \in A^+$ w.p. $p_i\propto \widehat P(\{i\})$
    \STATE Sample $k \not \in A^+$ w.p. $p_k \propto 1-\widehat P(\{i,k\})$.
    \item[] \textbf{return} $(A^+ \wo  \{j\}) \cup \{k\}$
  \end{algorithmic}
\end{algorithm}

This allows us to approximate true explicit negatives, as we use the empirical data to derive ``implausible'' sets. Note, however, that when using such negatives we have no guarantee that objective function will be well behaved, as opposed to the theoretically grounded dynamic negatives.

\subsection{Experimental setup}
The performance of all methods are compared using standard recommender system
metrics: Mean Percentile Rank (MPR). MPR is a recall-based
metric which evaluates the model's predictive power by measuring how well it
predicts the next item in a basket, and is a standard choice for recommender
systems~\citep{hu08,li10}.

Specifically, given a set $A$, let $p_{i,A}=\Pr(A\cup\{i\} \mid A)$. The percentile rank of an item $i$ given a set $A$ is defined as
\[\text{PR}_{i,A} = \frac {\sum_{i' \not\in A} \mathds 1(p_{i,A} \ge p_{i',A})} {|\Ycal \wo A|} \times 100\%\]
The MPR is then computed as\[\text{MPR}=\frac 1 {|\Tcal|} \sum_{A \in \mathcal T}\text{PR}_{i,A\wo \{i\}}\]
where $\Tcal$ is the set of test instances and $i$ is a randomly selected element in each set $A$. An MPR of 50 is equivalent to random selection; a MPR of 100 indicates that the model perfectly predicts the held out item. 

We also evaluate the discriminative power of each model using the AUC metric.
For this task, we generate a set of negative subsets uniformly at random.  For
each positive subset $A^+$ in the test set, we generate a negative subset $A^-$ of the
same length by drawing $|A^+|$ samples uniformly at random, while
ensuring that the same item is not drawn more than once for a subset.  We then
compute the AUC for the model on these positive and negative subsets, where the
score for each subset is the log-likelihood that the model assigns to the
subset. This task measures the ability of the model to discriminate between
positive subsets (ground-truth subsets) and randomly generated subsets.

In all experiments, 80\% of subsets are used for training; the remaining 20\%
served as test; convergence is reached when the relative change in the
validation log-likelihood is below a pre-determined threshold $\epsilon$, set
identically for all methods. All results are averaged over 5 learning trials.

\subsection{Amazon registries}
We conducted an experimental analysis on the largest 7 sub-datasets included in the Amazon Registry dataset, which has become a standard dataset for
\dpp modeling~\citep{gillenwater14,mariet15,gartrell17}.
Given the small size of these datasets (the largest has 100 items), these experiments serve only to provide insight into the general behavior of the baselines and CE methods as well as the influence of the hyperparameters on convergence.

Table~\ref{tab:runtimes} reports the average time to convergence for each method. As generating the dynamic negatives has a high complexity due to \dpp  conditioning, \dynCE is 2.7x slower than \expCE. \lowrank is the fastest method, as it does not need to process any negatives. \prodNCE is by far the most time-consuming.

\begin{table}[h]
  \centering
  \caption{\small Runtime to convergence (s) on the feeding Amazon registry ($\alpha=1$,  $|\mathcal A^-|/|\mathcal A^+|=0.5$,  $K=30$).}
  \label{tab:runtimes}
  \begin{center}
    \begin{sc}
      \begin{tabular}{lcccc}
        \toprule
        Method & \lowrank & \expCE & \dynCE & \prodNCE \\
        \midrule
        Runtime & 0.83 $\pm$ 0.54  & 2.69 $\pm$ 0.02  & 7.13 $\pm$ 0.28  &  27.59 $\pm$ 2.20 \\
        \bottomrule
      \end{tabular}
    \end{sc}
  \end{center}
\end{table}

We found that explicit and dynamic CE are not very sensitive to the $\alpha$ and $|\mathcal A^-|/|\mathcal A^+|$ hyperparameters. For this reason, in all further results, we set $\alpha=1$ and $|\mathcal A^-|/|\mathcal A^+|=.5$ in all further experiments. In previous work on low-rank \dpp learning~\citep{gartrell17}, $\alpha=1$ was found to be a reasonably optimal value, ensuring a fair comparison between all methods.

Further experiments reporting the MPR, AUC and various precisions for the Amazon registries are described in App.~\ref{app:abr}.

\subsection{UK and Belgian Retail Datasets}
\label{sec:belgian}
Following~\citep{gartrell16}, for both the UK and the Belgian dataset, we set the rank $K$ of the kernel to be the size of the largest subset in the dataset (K=100 for the UK dataset, K=76 for the Belgian dataset): this optimizes memory costs while still modeling all
ground-truth subsets. Based on our results on the smaller Amazon dataset, we fix
$|\Acal^-|/|\Acal^+|=0.5$ and $\alpha=1$.

Finally, corroborating our timing results on the Amazon registry, we saw that one iteration of \prodNCE required nearly 11 hours on the Belgian dataset (compared to 5 minutes for one iteration of CE). For this reason, we remove \prodNCE as a baseline from all remaining experiments, as it is not feasible in the general case.

Tables~\ref{tab:main} (a) and (b) summarize our results; the negative methods show
significant MPR improvement over \lowrank, with both \dynCE and \expCE
performing almost 10 points higher on the Belgian dataset, and 3 points higher on the UK dataset. This is a striking improvement, compounded
by small standard deviations confirming that these results are robust to matrix
initialization. 

We also see a dramatic improvement over \lowrank in AUC, with an improvement of
approximately 0.41 for the UK dataset and 0.37 for the Belgian dataset, across both \dynCE and \expCE methods.  Both \dynCE and \expCE perform
quite well, with an AUC score of approximately 0.9864 or higher for both models.
These results suggest that for larger datasets, CE can be effective at improving
the discriminative power of the \dpp.



\section{Conclusion and future work}
\vskip -5pt
We introduce the Contrastive Estimation (CE) optimization problem, which optimizes the difference of the traditional \dpp log-likelihood and the expectation of the \dpp model's log-likelihood under a \emph{negative} distribution $\nu$. This increases the \dpp's fit to the data while simultaneously incorporating inferred or explicit domain knowledge into the learning procedure.

CE lends itself to intuitively similar but theoretically different variants, depending on the choice of $\nu$: a static $\nu$ leads to significantly faster learning but allows spurious optima; conversely, allowing $\nu$ to evolve along with model parameters limits overfitting at the cost of a more complex optimization problem. Optimizing dynamic CE is in of itself a theoretical problem worthy of independent study.

Additionally, we show that low-rank \dpp conditioning complexity can be improved by a factor of $M$ by leveraging the dual representation of the low-rank  kernel. This not only improves prediction speed on a trained model, but allows for more efficient dynamic negative generation.

Experimentally, we show that CE with dynamic and explicit negatives provide comparable, significant improvements in the predictive performance of \dpps, as well as on the learned \dpp's ability to discriminate between real and randomly generated subsets.

Our analysis also raises both theoretical and practical questions: in particular, a key component of future work lies in better understanding how explicit domain knowledge can be incorporated into the generating logic for both dynamic and static negatives. Furthermore, the CE formulation in Eq.~\eqref{eq:2} suggests the possibility of using continuous labels for weighted samples within CE.


\paragraph{Acknowledgements.} This work was partially supported by a Criteo Faculty Research Award, and
NSF-IIS-1409802.

\bibliographystyle{icml2018}
\bibliography{bibliography}

\clearpage
\appendix
\section{Contrastive Estimation with the Picard iteration}
\label{app:picard}
Letting $\beta=|\Acal^+|-|\Acal^-| \ge 0$ and writing $\U_A$ as the $M \times |A|$ indicator matrix such that $\L_A = \U_A^\top\L \U_A$, we have
\begin{align*}
  \phi(\L) \propto& \underbrace{-\beta\log\det(\I+\X) + \sum_{A \in \Acal^+}  \log \det (\U_A^\top\X^{-1}\U_A)}_{f\text{ convex}}\\
                  &+ \underbrace{\beta\log \det(\X) - \sum_{A \in \Acal^-} \log \det (\U_A^\top\X^{-1}\U_A)}_{g\text{ concave}}
\end{align*}
where the convexity/concavity results follow immediately from~\citep[Lemma 2.3]{mariet15}. Then, the update rule $\nabla f(\L_{k+1}) = -\nabla g(\L_k)$ requires
\begin{align*}
  \beta& \L_{k+1} + \sum_{A \in \Acal^-}  \L_{k+1}\U_A(\U_A^\top\L_{k+1}\U_A)^{-1}\U_A^\top\L_{k+1} \\
  &\leftarrow  \beta (\I+\L_{k}^{-1})^{-1} + \sum_{A \in \Acal^+}  \L_k\U_A(\U_A^\top\L_k\U_A)^{-1}\U_A^\top\L_k
\end{align*}\vskip -1em
which cannot be evaluated due to the $\sum_{A \in \Acal^-}$ term.
\section{Amazon Baby registries experiments}
\label{app:abr}
\subsection{Amazon Baby Registries description}
\begin{table}[!h]
  \caption{Description of the Amazon Baby registries dataset.}
  \label{tab:abr}
  \vskip 0.15in
  \begin{center}
    \begin{small}
      \begin{sc}
        \begin{tabular}{lccc}
          \toprule
          Registry & $M$ & train size & test size \\
          \midrule
          health    & 62  & 5278  & 1320 \\
          bath      & 100 & 5510  & 1377 \\
          apparel   & 100 & 6482  & 1620 \\
          bedding   & 100 & 7119  & 1780 \\
          diaper    & 100 & 8403  & 2101 \\
          gear      & 100 & 7089  & 1772 \\
          feeding   & 100 & 10,090 & 2522 \\
          \bottomrule
        \end{tabular}
      \end{sc}
    \end{small}
  \end{center}
\end{table}

\subsection{Experimental results}
In Tab.~\ref{tab:main}(a), we compare the performance of the various algorithms with rank $K=30$. The regularization strength $\alpha$ is set to its optimal value for the \lowrank algorithm, and $|\Acal^-|/|\Acal^+|=1/2$. This allows us to compare the LR algorithm to its ``augmented'' negative versions without hyper-parameter tuning.  As \prodCE performs much worse than \lowrank, it is not included in  further experiments.

We evaluate the precision at $k$ as
\[p@k= \frac 1{|\mathcal T|} \sum_{A \in \mathcal T} \frac 1 {|A|} \sum_{i \in A} \mathds 1\big[\text{rank}(i \mid A \wo \{i\}) \le k\big].\]

\begin{table}[h]
  \caption{\small MPR, p@$k$, and AUC values for \lowrank, and baseline improvement over \lowrank for other methods.  Positive values indicate the algorithm performs better than \lowrank, and bold values indicate improvement over \lowrank that lies outside the standard deviation. Experiments were run 5 times, with  $|\mathcal A^+|/|\mathcal A^-|=\frac 12$; $\alpha$ is set to its optimal \lowrank value.} 
  \label{tab:main}
  \begin{center}
    \begin{small}
      \scalebox{0.85}{
        \begin{tabular}{ccccc}
        \toprule
        ~ & ~ & \multicolumn {3}{c}{Improvement over \lowrank} \\ \cline{3-5}\\[-.2cm]
          Metric & \lowrank & \dynCE & \expCE & \prodNCE \\
\midrule
          MPR & 70.50 & \bf 0.92 $\pm$ 0.56 & \bf 0.68 $\pm$ 0.62 & \bf 0.86 $\pm$ 0.55 \\
          p@1 & 9.96 &  0.67 $\pm$ 0.75 &  0.58 $\pm$ 0.76 &  0.20 $\pm$ 1.75 \\
          p@5 & 25.36 & \bf 1.04 $\pm$ 0.82 & \bf 0.78 $\pm$ 0.67 &  0.67 $\pm$ 1.09 \\
          p@10 & 36.50 & \bf 1.39 $\pm$ 0.85 & \bf 1.13 $\pm$ 0.79 &  0.97 $\pm$ 1.18 \\
          p@20 & 51.22 & \bf 1.38 $\pm$ 0.97 & \bf 1.28 $\pm$ 1.11 & \bf 1.35 $\pm$ 1.20 \\
          AUC & 0.630 & \bf 0.027 $\pm$ 0.017 & \bf 0.026 $\pm$ 0.016 &  0.009 $\pm$ 0.017 \\
        \bottomrule
        \end{tabular}
      }
    \end{small}
  \end{center}
\end{table}

Compared to traditional SGA methods, algorithms that use inferred negatives perform (\prodCE excepted) better across all metrics and datasets. 
\dynCE and \expCE provide consistent improvements compared to the other methods, whereas \prodNCE shows a higher variance and slightly worse performance.
Improvements observed using \dynCE and \expCE are larger than the loss in performance due to going from full-rank to low-rank kernels reported in~\citep{gartrell17}.



Finally, we also compared all methods when tuning both the regularization $\alpha$ and the negative to positive ratio ${|\Acal^-|}/{|\Acal^+|}$, but did not see any significant improvements. As this suggests there is no need to do additional hyper-parameter tuning when using CE, we fix $\frac{|\Acal^-|}{|\Acal^+|} = \frac 12$ for all experiments.

\end{document}